\documentclass{article}

\usepackage[preprint,nonatbib]{nips_2018}

\usepackage{xcolor}
\usepackage{soul}
\usepackage[utf8]{inputenc}
\usepackage[small]{caption}

\usepackage{enumerate}
\usepackage{amsmath}
\usepackage{amsthm}
\usepackage{amssymb}
\usepackage{algorithm}
\usepackage{algorithmic}
\usepackage{xcolor}
\usepackage{tikz}
\usetikzlibrary{calc}
\usepackage{hyperref}

\usepackage{url}

\newtheorem{thm}{Theorem}
\newtheorem{lmm}{Lemma}
\newtheorem{crl}{Corollary}

\newcommand{\E}{\ensuremath{\mathop{\mathbb{E}}}}

\title{A Renewal Model of Intrusion}

\author{David Tolpin, dvd@offtopia.net}

\begin{document}

\maketitle

\begin{abstract}
    We present a probabilistic model of an intrusion in a 
    renewal process. Given a process and a sequence of events,
    an intrusion is a subsequence of events that is not produced
    by the process. Applications of the model are, for example,
    online payment fraud with the fraudster taking over a user's
    account and performing payments on the user's behalf, or
    unexpected equipment failures due to unintended use. 

    We adopt Bayesian approach to infer the probability of an
    intrusion in a sequence of events, a MAP subsequence of 
    events constituting the intrusion, and the marginal
    probability of each event in a sequence to belong to the
    intrusion. We evaluate the model for intrusion detection on
    synthetic data and on anonymized data from an online
    payment system.
\end{abstract}

\section{Overview}

We approach the problem of distinguishing between `native'
events and intrusions in an event stream arriving over time.
This problem arises in multiple applications. Consider, for
example, an online payment service where the users connect with
their credentials and pay for goods or services. A thief can
illegally obtain access to another user's account and steal
money by sending payments on behalf of the legitimate user.  Is
there a way to identify illegal payments by looking at a
sequence of payments even if each individual payment looks
legitimate?

We turn to the \textit{renewal process} as the basis for
the probabilistic generative model of the problem. Renewal
processes~\cite{C70} are used to model arrival of events where hold times
between events are independently distributed. We assume that
regular events come from a renewal process with known
parameters, which can be given or inferred. We then
consider a sequence of recent events and reason about the
likelihood that some of the events are `foreign' rather than
belong to the process.  We adopt the Bayesian approach to infer
the probability of an intrusion in a given event sequence, a
maximum \textit{a posteriori} probability (MAP) subsequence of
events constituting the intrusion, and the marginal probability
of each event to belong to the intrusion.

We show that the inference can be performed in polynomial time.
We implement the inference algorithms and evaluate the inference
on synthetic data and on anonymized data from an online
payment system. 

\subsection*{Contributions}

The paper brings the following contributions:

\begin{itemize}
    \item A probabilistic generative model for inference about
        intrusions in renewal processes.
    \item Polynomial-time algorithms for computing the
        probability of an intrusion, a MAP subsequence of events
        constituting the intrusion, and the marginal probability
        of each event to belong to the intrusion.
    \item An evaluation of applicability of the algorithms to
        intrusion detection in online payment systems.
\end{itemize}

\section{Related Work}

The problem of detecting an intrusion in a sequence of events
belongs to the field of anomaly (or novelty) detection. An
extensive review of novelty detection  in general is provided
in~\cite{MS03,PCC+14}. Anomaly detection in discrete sequences
is reviewed in~\cite{CBK12}, and in temporal data
in~\cite{GGA+14}.

A generative probabilistic model~\cite{E84} is used to reason
about intrusion probabilities. Much of the recent fundamental
and applied research on unsupervised learning in general and
anomaly detection in particular involves generative
probabilistic models~\cite{RG99,PER00,XPS+11,XPS11}.

A renewal process is a discrete stochastic
process~\cite{G14}. Discrete stochastic processes arise as
models in many applications~\cite{S09,G14,SP13}. Depending on the
nature of the phenomenon being modelled, different discrete
stochastic processes are used, such as Poisson
processes~\cite{LM03,SJ10}, Cox processes~\cite{L98},
interacting point processes and in particular Hawkes
processes~\cite{O88,CM12}, Markov processes~\cite{Y20}, and
other variations~\cite{WP96,TA06}.

The present work differs from earlier research in the following
aspects:
\begin{enumerate}[a)]
    \item A specific type of novelty, namely an intrusion, is
        considered. The sequence of events is viewed as a mixture
        of normal activity and an intrusion.
    \item The generative model is used to predict both the
        probability of an intrusion  and the marginal
        probability of each event to belong to the intrusion
        (rather than just the probability of an intrusion).
    \item No prior assumption is made about the stochastic
        process realised by the intrusion.
\end{enumerate}

\section{Preliminaries}

A renewal process~\cite{C70}\cite[Chapter~5]{G14} is a
generalization of the Poisson point process.
In a renewal process, the interarrival intervals $\Delta t$ are
non-negative, independent, and identically distributed
random variables. A renewal process can be characterized in
several ways --- by the distributions of either arrival times,
interarrival intervals, or the number of arrivals during a unit
time interval. In this paper, we characterize renewal processes
by the distribution of interarrival intervals. We write
\begin{equation}
    \Delta t \sim F(\theta_F)
    \label{eq:renewal-process}
\end{equation}
to describe a renewal process with interarrival intervals
drawn from distribution $F$ with parameter $\theta_F$. For
example, the Poisson process is a renewal process with
exponentially distributed interarrival intervals:
\begin{equation}
    \Delta t_{Poisson} \sim \mathrm{Exponential}(\lambda)
    \label{eq:poisson-process}
\end{equation}
Renewal processes are used as a simple model for systems that
repeatedly return to a state probabilistically equivalent to the
initial state. 

\section{Probabilistic Generative Model of Intrusion}
\label{sec:probmod}

In the problem of \textit{intrusion detection in a
renewal process} we are \textbf{given:}

\begin{itemize}
        \item A renewal process
            \begin{equation*}
                \Delta t \sim F(\theta_F)
            \end{equation*}
        \item A prior probability $p_\epsilon$ that an individual
            event in the sequence belongs to the intrusion.
        \item A time interval $[t_s, t_e]$, of duration $T = t_e - t_s$.
        \item A sequence $S$ of $N$ events $\{t_1, t_2, \dots, t_N\}$
            within the time interval, i.e. $t_s \le t_1 \le t_2 \le
            \dots \le t_N \le t_e$.
\end{itemize}
Based on this, we \textbf{need to determine:}
\begin{enumerate}
    \item The posterior probability of an intrusion in the
        sequence.
    \item Maximum \textit{a posteriori} subsequence $I_{MAP}$ of events
        constituting the intrusion.
    \item The marginal probability of each event to belong to
        the intrusion.
\end{enumerate}

To solve the problem, we construct a generative model that
produces a sequence of events, taking the possibility of an
intrusion into account, and then perform posterior inference on
the model. There are two essential observations~\cite{C70}:
\begin{enumerate}
    \item A renewal process is infinite in both directions.
    \item The probability density of the interarrival
        interval of a renewal process is fully determined by
        the time interval passed from the last event.
\end{enumerate}
Based on these observations, just two more events --- $t_0$ and 
before the first event and $t_{N+1}$ after the last event in the sequence ---
fully define the context of the given sequence of events; also,
the times can be shifted arbitrarily by the same offset, e.g. so
that the earliest event takes place at time 0, $t_0 = 0$. Hence,
the generative model must draw the number $K$ of events
belonging to the intrusion, $0 \le K \le N$, and then generate
$N - K + 2$ events from the renewal process, starting with an
event at time 0 (Algorithm~\ref{alg:parametric}).

\captionof{algorithm}{Generative model of an intrusion in a
                      renewal process}
\label{alg:parametric}
\begin{algorithmic}[1]
    \STATE $K \sim \mathrm{Binomial}(N, p_\epsilon)$
    \STATE $t_0 \gets 0$
    \FOR {$i = 1$ \textbf{to} $N - K$}
        \STATE $\Delta t_i \sim F(\theta_F)$
        \STATE $t_i \gets t_{i-1} + \Delta t_i$
    \ENDFOR
    \STATE $\Delta t_{N - K + 1} \sim F(\theta_F)$
    \STATE $t_{N - K + 1} \gets t_{N - K} + \Delta t_{N - K + 1}$
\end{algorithmic}

The model implies that the renewal process is fully known.  
Section~\ref{sec:param} discusses ways to
estimate the process parameter $\theta_F$.

\section{Posterior Inference}
\label{sec:posterior}

In what follows,  $f(\cdot)$ stands for the probability density of
distribution $F$, $\widetilde \Pr(\cdot)$ stands for unnormalized
probability. For brevity, we drop explicit conditioning of
probabilities on problem parameters. The proofs are provided in
the supplementary material.

\subsection{Probability of Intrusion Subsequence}
\label{sec:prob-i}

The posterior inference in subsections~\ref{sec:post-map}
and~\ref{sec:post-ip} involves computing the unnormalized probability
of intrusion subsequence $I$ given $S$.

\begin{lmm} Let $\{k_1, k_2, \dots, k_{N - K}\}$ be the
    indices of events in $S \setminus I$, as produced by
    Algorithm~\ref{alg:parametric} (for $N=K$, let us
    set $k_1 = N+1$, $k_{N-K} = k_0 = 0$). Then (see
    Figure~\ref{fig:prob-i}),

    \begin{align}
        \label{eqn:prob-i}
         \widetilde \Pr&(I| S) = P_{k_1} \!\!\!\!\! \prod_{j=1}^{N-K-1}\!\!\!\!\!Q_{k_j, k_{j+1}} R_{k_{N-K}}, \\ \nonumber
        \mbox{where}& \\ \nonumber
         P_k &={\E_{\tau \sim F}\left[f(\tau)\right]}^{-1} \cdot
        \begin{cases}
            p_\epsilon^N   \frac {\int_{T}^\infty (\tau - T)f(\tau)^2 d\tau} {\int_{T}^\infty (\tau - T)f(\tau) d\tau}  & \mathbf{if}\; k = N + 1, \\
            p_\epsilon^{k - 1} \frac {\int_{t_k - t_s}^\infty f(\tau)^2d\tau} {\int_{t_k - t_s}^\infty f(\tau)d\tau} & \mathbf{otherwise.}
        \end{cases} \\ \nonumber
         Q_{k_1,k_2}& = {\E_{\tau \sim F}\left[f(\tau)\right]}^{-1} \cdot (1-p_\epsilon)p_\epsilon^{k_2 - k_1 - 1} {f(t_{k_2} - t_{k_1})}, \\ \nonumber
         R_k& = {\E_{\tau \sim F}\left[f(\tau)\right]}^{-1} \cdot 
        \begin{cases}
            {\E_{\tau \sim F}\left[f(\tau)\right]} & \mathbf{if}\; k = 0, \\
            (1-p_\epsilon)p_\epsilon^{N - k} 
            \frac {\int_{t_e - t_k}^\infty f(\tau)^2d\tau}{\int_{t_e - t_k}^\infty f(\tau)d\tau} & \mathbf{otherwise.}
        \end{cases}
    \end{align}
    \label{lmm:prob-i}

    Factors $P_k$, $Q_{k1, k2}$, and $R_k$ (Figure~\ref{fig:prob-i}) correspond to transitions in the generative model (Algorithm~\ref{alg:parametric})
    \begin{itemize}
        \item from the extra event at the beginning to the first event in $S \setminus I$,
        \item between events within $S \setminus I$, and
        \item from the last event in $S \setminus I$ to the extra
            event at the end.
    \end{itemize}
\end{lmm}

\begin{figure}[t]
    \begin{tikzpicture}[->,every node/.style={transform shape},auto,node distance=36pt]
        \tikzstyle{event}=[shape=circle]
        \tikzstyle{process}=[event,fill=darkgray,draw=black,text=white]
        \tikzstyle{extra}=[event,fill=lightgray,draw=darkgray,style=dashed,text=black]
        \tikzstyle{intrusion}=[event,fill=red,draw=black,text=black]

        \node[extra] (e0) {$e_0$};
        \node[intrusion] (e1) [right of=e0] {$e_1$};
        \node[intrusion] (e2) [right of=e1] {$e_2$};
        \node[process] (e3) [right of=e2] {$e_3$};
        \node[intrusion] (e4) [right of=e3] {$e_4$};
        \node[process] (e5) [right of=e4] {$e_5$};
        \node[intrusion] (e6) [right of=e5] {$e_6$};
        \node[intrusion] (e7) [right of=e6] {$e_7$};
        \node[process] (e8) [right of=e7] {$e_8$};
        \node[extra] (e9) [right of=e8] {$e_9$};
        \path (e0) edge [bend right] node {$P_3$} (e3);
        \path (e3) edge [bend left] node {$Q_{3,5}$} (e5);
        \path (e5) edge [bend right] node {$Q_{5,8}$} (e8);
        \path (e8) edge [bend left] node {$R_8$} (e9);
    \end{tikzpicture}
    \caption{The probability of an intrusion subsequence. Process events ($S
        \setminus I$) are \textbf{black}, intrusion events ($I$) are
        \textcolor{red}{\textbf{red}}, extra
        events at the beginning and at the end are
        \textcolor{gray}{\textbf{gray}} and dashed.}
    \label{fig:prob-i}
\end{figure}
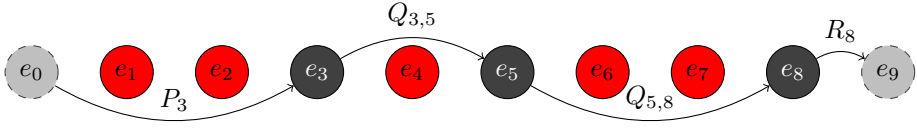

\subsection{Maximum \textit{a Posteriori} Subsequence}
\label{sec:post-map}

In (\ref{eqn:prob-i}) each of factors $P_k$, $Q_{k_1, k_2}$, $R_k$
(Equation~\ref{eqn:prob-i}) is independent of the rest of events
given two adjacent process events. Therefore, finding a
MAP subsequence of intrusion events can be formulated as the
shortest path problem in a directed acyclic graph.

\begin{thm} Let us construct a weighted directed acyclic graph
    $G=(V, E)$ where the set of vertices $V$ is the set of
    events $\{e_0, e_1, \cdots, e_{N+1}\}$, and the set of edges
    $E$ contains a weighted edge for each pair of vertices, from the smaller
    to the greater index: $\{(e_{k_1}, e_{k_2}), w_{k_1, k_2}
    \,|\, k_2 > k_1\}$. The edge weights $w_{k_1, k_2}$ are:
    \begin{equation}
        w_{k_1, k_2} = 
        \begin{cases}
            - \log P_{k_2}&\mathbf{if}\;k_1=0, \\
            - \log R_{k_1}&\mathbf{if}\;k_2=N+1, \\
            - \log Q_{k_1,k_2}&\mathbf{otherwise.}
        \end{cases}
        \label{eqn:post-map-weights}
    \end{equation}
    
    Then, a MAP subsequence of events $I_{MAP}$ is the set of
    events in a shortest path from $e_0$ to $e_{N+1}$ with
    the extra events $\{e_0, e_{N+1}\}$ removed.
    \label{thm:post-map}
\end{thm}

Theorem~\ref{thm:post-map} implies that a MAP subsequence can be
computed in time $\Theta(N^2)$.

\subsection{Intrusion Probabilities}
\label{sec:post-ip}

The probability of an intrusion $\Pr(I \ne \emptyset|S)$ and the
marginal probability $\Pr(e_k  \in I | S)$ of each event
$e_k$, $\forall k \in \{1, \dots, N\}$, to belong to the intrusion
can be computed in polynomial time.  We first present an
algorithm for computing the posterior probability of intrusion.
Then, we show how the same algorithm can be generalized to also
compute the marginal probability of any given event to belong to
the intrusion.  Finally, we introduce an algorithm for computing
simultaneously, and hence more efficiently, the probability of
an intrusion and the marginal probability of each event in $S$
to belong to the intrusion.

The algorithms compute the probabilities according
to~(\ref{eqn:post-ip-bayes-ip}) and~(\ref{eqn:post-ip-bayes-mp}):
\begin{align}
    \label{eqn:post-ip-bayes-ip}
    \!\!\!\Pr(I \ne \emptyset | S) & = \!\! 1\! - \! \Pr(I = \emptyset | S) \!\! = \!\! 1\! -\! \frac {\widetilde \Pr(I = \emptyset|S)} {\widetilde \Pr(I \in 2^S | S)}  \\
    \label{eqn:post-ip-bayes-mp}
    \!\!\!\Pr(e_k\! \in\! I | S)  & = \!\! 1\! -\! \Pr(e_k\! \notin\! I | S)\!\!   = \!\!  1\! - \! \frac {\widetilde \Pr(e_k \!\notin\! I|S)} {\widetilde \Pr(I \in 2^S | S)} 
\end{align}

Both equations involve computing the unnormalized marginal likelihood $\widetilde \Pr(I \in 2^S | S)$.  Lemma~\ref{lmm:post-ip-ml-alg} gives an algorithm for computing $\widetilde \Pr(I \in 2^S | S)$ in polynomial time.

\begin{lmm}
    Algorithm~\ref{alg:post-ip-ml} computes $\widetilde \Pr(I \in 2^S | S)$ in time
    $\Theta(N^2)$. 
    \label{lmm:post-ip-ml-alg}
\end{lmm}

\captionof{algorithm}{Marginal likelihood of $S$}
\label{alg:post-ip-ml}
\begin{algorithmic}[1]
    \STATE $k \gets 1$
    \FOR {$k = 1$ \textbf{to} $N$} \label{alg:post-ip-ml-loop-start}
        \STATE $a_k \gets P_{k}$ \label{alg:post-ip-ml-s}
        \FOR {$j = 1$ \textbf{to} $k - 1$}
            \STATE $a_k \gets a_k + a_j Q_{j,k}$ \label{alg:post-ip-ml-j}
        \ENDFOR
    \ENDFOR \label{alg:post-ip-ml-loop-end}
    \STATE $A \gets P_{N+1}$ \label{alg:post-ip-ml-e-start}
    \FOR {$j = 1$ \textbf{to} $N$}
        \STATE $A \gets A + a_jR_j$
    \ENDFOR \label{alg:post-ip-ml-e-end}
    \RETURN $A$
\end{algorithmic}

Being able to compute the unnormalized marginal likelihood, we can
immediately obtain the intrusion
probability~(\ref{eqn:post-ip-bayes-ip}). Marginal
probabilities~(\ref{eqn:post-ip-bayes-mp}) involve the unnormalized
probability of a particular event $e_{k^*}$ not in $I$ given $S$,
which can be computed similarly to the unnormalized marginal likelihood.
However, much of the computation would be reused between
different marginal probabilities; in particular, the
computations for two events $e_{k_1}, e_{k_2}$ are the same
until $\min(k_1, k_2)$.  Theorem~\ref{thm:post-ip-all} gives an
algorithm\footnote{Algorithm~\ref{alg:post-ip-all} bears
similarity to the forward-backward algorithm for Markov
chains~\cite{BPS+70}, but computes marginal probabilities of
occurrence of a node in the sequence rather than of states in
the sequence of nodes.} that computes the intrusion probability and
all posterior probabilities simultaneously, reusing computations.

\begin{thm}
    Algorithm~\ref{alg:post-ip-all} computes $\Pr(I \ne \emptyset | S)$
    and $\Pr(e_k \in I | S)\; \forall k \in \{1, \dots, N\}$ in
    time $\Theta(N^2)$.

    \label{thm:post-ip-all}
\end{thm}

\captionof{algorithm}{Intrusion probability and marginal probabilities of events}
\label{alg:post-ip-all}
\begin{algorithmic}[1]
    \STATE -{}- {Run Algorithm~\ref{alg:post-ip-ml} forward}
    \STATE $\widetilde \Pr(I\in 2^S|S) \gets$ Algorithm~\ref{alg:post-ip-ml}
    \STATE -{}- {Compute the intrusion probability}
    \STATE $\widetilde \Pr(I = \emptyset|S) \gets$ Equation~\ref{eqn:prob-i}\mbox{\bf{ where }}$I=\emptyset$
    \STATE $\Pr(I \ne \emptyset|S) = 1 - \frac {\widetilde \Pr(I = \emptyset|S)} {\widetilde \Pr(I \in 2^S | S)}$
    \label{alg:post-ip-all-ip}
    \FOR {$k = 1$ \textbf{to} $N$} 
        \STATE $a^f_k \gets a_k$ -{}- {Store $a_i$ from the forward run}
    \ENDFOR
    \STATE -{}- {Compute $S'$, $t'_s$, $t'_e$ by reversing the time}
    \STATE $t'_s, t'_e = -t_e, -t_s$
    \FOR {$k = 1$ \textbf{to} $N$} 
        \STATE $t'_k, y'_k = -t_{N - k + 1}, y_{N - k + 1}$
    \ENDFOR
    \STATE -{}- {Run Algorithm~\ref{alg:post-ip-ml} backward}
    \STATE Algorithm~\ref{alg:post-ip-ml} \textbf{where} $S=S',\,t_s=t'_s,\,t_e=t'_e$
    \FOR {$k = 1$ \textbf{to} $N$} 
        \STATE $a^b_{N - k + 1} \gets a_k$ -{}- {Store $a_i$ from the backward run}
    \ENDFOR
    \STATE -{}- {Compute marginal probabilities}
    \FOR {$k= 1$ \textbf{to} $N$} \label{alg:post-ip-all-loop-start}
        \STATE $\widetilde \Pr(e_k \notin I | S) = \frac {a_k^fa_k^b} {1 - p_\epsilon}$\label{alg:post-ip-all-once}

        \STATE $\Pr(e_k \in I | S) = 1 - \frac {\widetilde \Pr(e_k \notin I | S)} {\widetilde \Pr(I \in 2^S | S)}$ \label{alg:post-ip-all-mp}
    \ENDFOR \label{alg:post-ip-all-loop-end}
    \RETURN $\Pr(I \ne \emptyset | S),\;\Pr(e_k \in I|S)\,\forall k$
\end{algorithmic}

\section{Process Parameters}
\label{sec:param}

The results in Section~\ref{sec:posterior} rely on the process
parameters being known. 
Subsections~\ref{sec:param-past}--\ref{sec:param-bayes} discuss
ways in which the parameters can be estimated.

\subsection{Estimation from Past Data}
\label{sec:param-past}

The most straightforward approach is to estimate the parameters
from the past data under the assumption that the data do not
contain any intrusions.  This assumption is adequate either if 
intrusions are detected and removed from the data, or if they
are rare, such that their influence on estimation of the process
parameters is negligible.

\subsection{Maximum Likelihood Estimation by Ex\-pect\-ation-Maxi\-mization}
\label{sec:param-em}

The process parameters can be chosen to maximize the likelihood
of the MAP subsequence. This yields an expectation-maximization
(EM) algorithm (Algorithm~\ref{alg:param-em}) alternating between
finding the MAP subsequence $I_{MAP}$ of intrusion events and
estimating parameters from the remaining subsequence $S
\setminus I_{MAP}$.

\captionof{algorithm}{Estimating process parameters by expectation-maximization}
\label{alg:param-em}
\begin{algorithmic}[1]
    \STATE $I_{MAP}^{prev} \gets \emptyset$ \label{alg:param-em-init}, $i \gets 1$
    \LOOP
        \STATE \underline{M step:} estimate $\theta_F, \theta_G$ from $S \setminus I_{MAP}^{prev}$ \label{alg:param-em-estimate}
        \IF {$i = N_{iter}$}
            \STATE \textbf{break} \label{alg:param-em-maxiter}
        \ENDIF
        \STATE \underline{E step:} compute $I_{MAP}$
        \IF {$I_{MAP} = I_{MAP}^{prev}$}
            \STATE \textbf{break} \label{alg:param-em-stuck}
        \ELSIF  {$|I_{MAP}| > K_{max}$}
            \STATE \textbf{break} \label{alg:param-em-kmax}
        \ENDIF
        \STATE $I_{MAP}^{prev} \gets I_{MAP}$, $i \gets i + 1 $
    \ENDLOOP
    \RETURN $\theta_F, \theta_G$
\end{algorithmic}

The initial parameter values are set under the assumption that
there is no intrusion, i.e. from the whole sequence $S$
(line~\ref{alg:param-em-init}).  Given a sequence of events, the
parameters are estimated as in Subsection~\ref{sec:param-past}
(line~\ref{alg:param-em-estimate}).  The algorithm terminates
either when $I_{MAP}$ stays the same in two subsequent
iterations (line~\ref{alg:param-em-stuck}), thus reaching a
fixed point, or after a pre-defined maximum number of iterations
$N_{iter}$ (line~\ref{alg:param-em-maxiter}).

A pitfall of this EM scheme is that the process parameters
cannot be estimated reliably if $S \setminus I_{MAP}$ becomes
too small.  Hence, the algorithm must also be interrupted when
the size of $I_{MAP}$ exceeds a certain
threshold $K_{max}$ (line~\ref{alg:param-em-kmax}).

There is no general guarantee that an EM algorithm converges to
the global maximum~\cite{W83}. In practice, however,
Algorithm~\ref{alg:param-em} works well for sufficiently small
values of $p_\epsilon$ (see Section~\ref{sec:empirical} for empirical
evidence), which is often the case in intrusion detection
applications. 

\subsection{Bayesian Inference of Posterior Distribution of Parameters}
\label{sec:param-bayes}

In the Bayesian setting, a prior can be imposed upon the process
parameters.  The posterior inference is performed on the joint
distribution of the process parameters conditioned on the
marginal likelihood of $S$ (Algorithm~\ref{alg:post-ip-ml}). A
drawback of this approach is that the inference may be too
expensive computationally. As the problem of detecting
intrusions in online event streams often arises in settings that
require fast response, maximum-likelihood estimation from past
data (Subsection~\ref{sec:param-past}) or from the given event
sequence (Subsection~\ref{sec:param-em}) may be a better choice.

\section{Empirical Evaluation}
\label{sec:empirical}

In the case studies that follow we evaluate the algorithms of
Sections~\ref{sec:posterior} and~\ref{sec:param} on both
synthetic and real-world data. Evaluation on synthetic data
provides an evidence that the algorithms work on data generated
by a renewal process. Evaluation on
real-world data examines performance of the algorithms when the
properties of the generating process are unknown, as well as
assesses their applicability to practical intrusion detection.

The data, the algorithms, and the code to run the experiments
are available at \url{https://github.com/dtolpin/rmi-case-studies}.

\subsection{Evaluation on Synthetic Data}

We generate data from a renewal process with Gamma-distributed
interarrival intervals for shapes 1, 2, 4, and 8.  The dataset
is balanced so that a half of the dataset entries contains an
intrusion. Intrusion events are uniformly distributed over a
subinterval of each entry with intrusion, chosen uniformly with
average length of $\frac 1 3$ of the total entry duration.
10\,000 entries of 20 events are generated for each intrusion probability. 

\begin{figure}[h]
    \centering
    \includegraphics[width=0.8\textwidth]{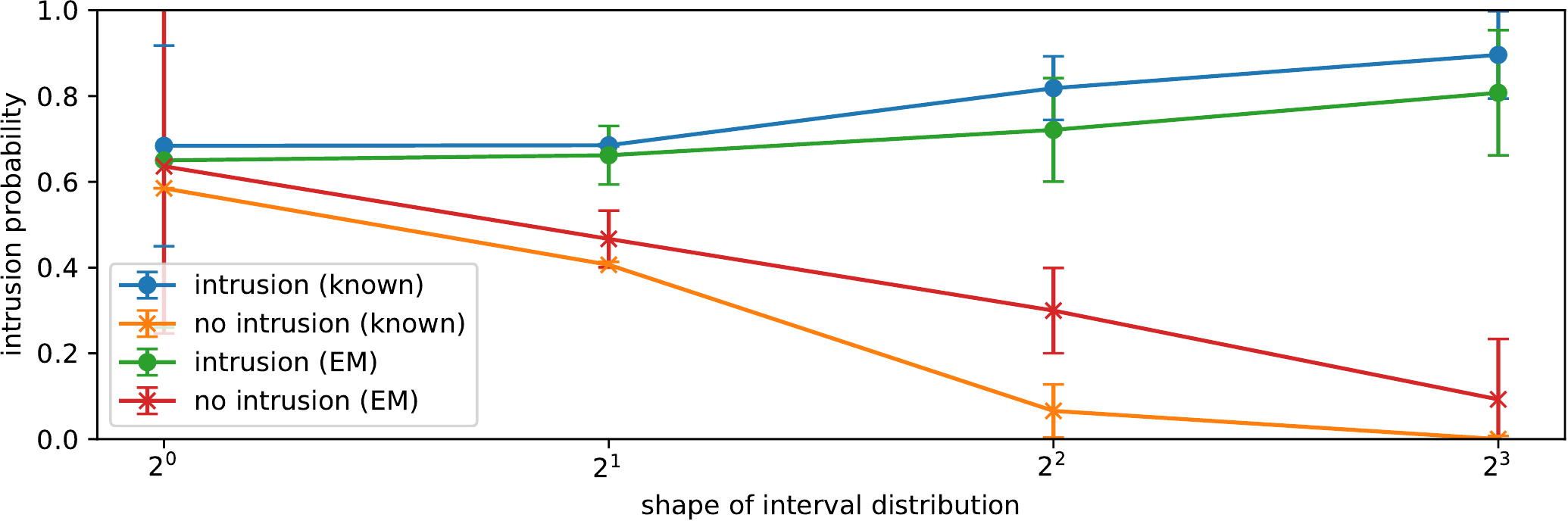}
	\caption{Posterior intrusion probability. The error
    bars are for two standard deviations.}
    \label{fig:synthetic-alert-score}
\end{figure}
Figure~\ref{fig:synthetic-alert-score} shows average posterior
intrusion probability as a function of the shape of interval
distribution for both negative and positive entries. The
probability is computed either for known process parameters, or
for parameters estimated through the EM algorithm
(Section~\ref{sec:param-em}).  The posterior intrusion
probability for positive and negative samples differs
sufficiently for shapes greater than 1 to reliably distinguish
between samples with and without intrusion in both. 

Shape 1 corresponds to the Poisson process. In  a Poisson process the
joint density of a sequence of events is independent of
intermediate intervals given the interval between the last and
the first event. Hence, intrusion probabilities in positive and
negative samples are close to 0.5 and to each other. 

\begin{figure}[h]
    \centering
    \includegraphics[width=0.8\textwidth]{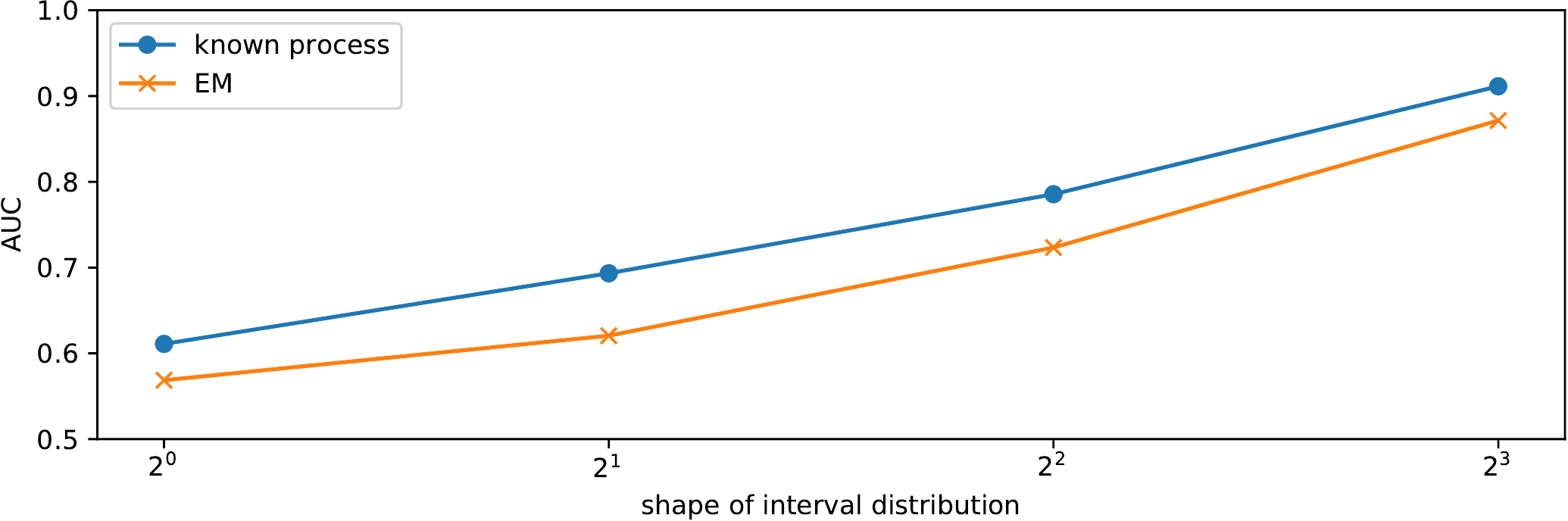}
    \caption{Per-entry area under the ROC curve of intrusion.}
    \label{fig:synthetic-auc-per-entry}
\end{figure}
Figure~\ref{fig:synthetic-auc-per-entry} shows \textit{area
under the ROC curve} (AUC) as a function of the prior intrusion
probability for each combination of data and algorithm
parameters. AUC reflects the classification accuracy for all
combinations of false negative and false positive rates. When
interarrival intervals are used for intrusion detection, for
both known process parameters and parameters estimated by the EM
algorithm, AUC stays above 0.6 for shapes greater than 1, with
the highest values of $\approx 0.9$.

\begin{figure}[h]
    \centering
    \includegraphics[width=0.8\textwidth]{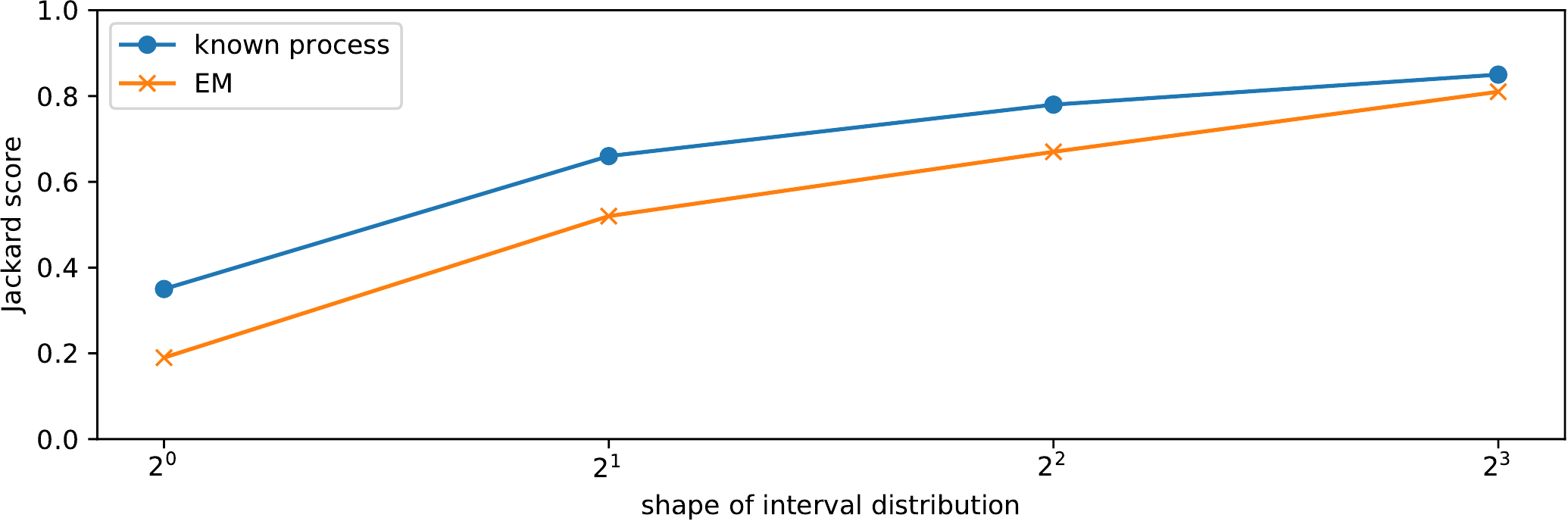}
    \caption{Jaccard similarity score.}
    \label{fig:synthetic-jaccard-score}
\end{figure}
Figure~\ref{fig:synthetic-jaccard-score} shows average
Jaccard similarity score between the MAP intrusion subsequence
(Section~\ref{sec:post-map}) and the actual intrusion. Jackard
similarity score stays above $\approx 0.5$ for shapes greater
than 1. The score is low for shape 1, because the Poisson process 
implies that in the presence of an intrusion any subsequence of
intermediate events of given size has the same probability to
belong to the intrusion.

\subsection{Evaluation on Anonymized Real-World Data}

We obtained anonymized data from an online payment system,
consisting of 1000 log fragments. The data contains time stamps
and amounts of payments. The
data is anonymized in the following way. Each entry (log
fragment) contains 50 events. The event times are rescaled
so that the events fall within interval $[0, 1]$. The payment
amounts are normalized to have the mean of 1.

The renewal process model of intrusion is straightforwardly
extended to a renewal process with independent marks by multiplying
each term in~(\ref{lmm:prob-i}) by the corresponding mark
density. We evaluate intrusion detection based on intervals
alone, marks alone as a baseline, and marks and intervals combined.

Neither parameters of normal processes generating the events nor
the prior intrusion probability are known.  To estimate the
prior intrusion probability, we split the dataset into the
training (20\%) and test (80\%) datasets.  We choose the
probability to maximize AUC on the training dataset, and then
run the inference on the training dataset. For both training and
test dataset, we estimate process parameters with the EM
algorithm (Section~\ref{sec:param-em}).

\begin{center}
{\small
\setlength\tabcolsep{4pt}
\begin{tabular}{l | c | c | c}
                 & Intervals  & Marks   & Marks and intervals \\ \hline
 AUC             & 0.691     & 0.584      & 0.733 \\
 Jaccard score   & 0.686     & 0.532      & 0.725
\end{tabular}}
\vspace{-6pt}
\end{center}
\captionof{table}{Metrics of intrusion detection in anonymized
    payment data. Marks only yield much lower accuracy than
    either intervals and marks or intervals only.}
\label{tbl:anonymized}

\vspace{6pt}

Intrusion detection metrics on the test dataset are shown in
Table~\ref{tbl:anonymized}. Detection based on marks only serves
as a comparison baseline. While marks
alone provide some information about intrusion, the detection
accuracy is much higher when interarrival intervals are taken
into account through the renewal process model.

\begin{figure}[h]
    \centering
    \includegraphics[width=0.8\textwidth]{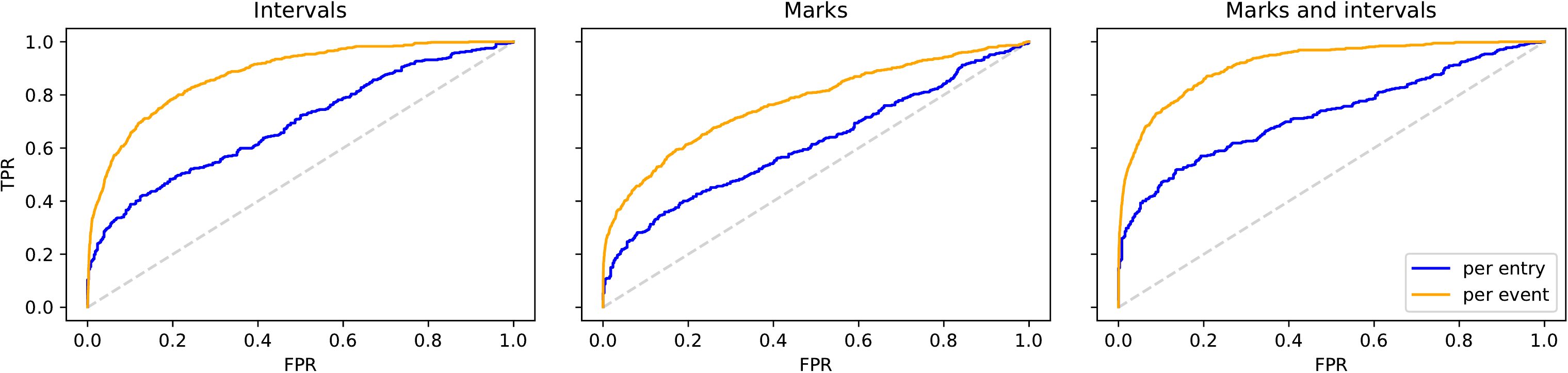}
	\caption{The ROC curve of intrusion detection in anonymized
	payment data.}
	\label{fig:anonymized-roc}
\end{figure}
Figure~\ref{fig:anonymized-roc} shows the ROC curves of intrusion
detection on the test dataset. According to the curves, when the
renewal process is used for detection, $\approx
70\%$ of entries with intrusion are among $\approx 30\%$ topmost
entries ordered by posterior intrusion probability. $\approx
80\%$ of intrusion events are among $\approx 10\%$ topmost
events ordered by posterior marginal probability of belonging to
an intrusion. Compared to that, when the detection is based on
marks alone, the intrusion detection accuracy is only slightly
better than random guess.

\section{Discussion}

We introduced a probabilistic generative model for inference
about intrusions in renewal processes. Posterior
inference in this model can be performed in polynomial time to
obtain the posterior intrusion probability, the marginal
probability of each event to belong to an intrusion, and a MAP 
subsequence of intrusion events. When process parameters are
unknown, they can be efficiently estimated using an
expectation-maximization algorithm.

We evaluated the inference algorithms, including parameter
estimation, on both synthetic and anonymized real-world data.
In both cases the inference algorithms yielded results
suggesting their suitability for intrusion detection. Due to
low runtime complexity, the algorithms suit well online
applications, such as fraud detection in online payment systems.

Application of the algorithms is based on the assumption that
the process generating normal events is sufficiently well
described by a renewal process. Evaluation on the
anonymized real-world data from an online payment system 
supports feasibility of this model. However, one may envision
cases where renewal process is inadequate, such that when
multiple past events affect the distribution of future event
times and marks. In such cases, a model based on interacting
point processes, in particular on the Hawkes process, should
be considered; however, exact or approximate inference
algorithms in such a model may have higher computational 
complexity. On the other hand, if the event series are well 
described by a Poisson process, intrusions cannot be
reliably identified based on interarrival intervals.  


\bibliographystyle{plain}
\bibliography{refs}

\clearpage

\section{PROOFS}

Lemma~\ref{lmm:prob-i}

\begin{proof}
    To define $\Pr(I|S)$ for any $I \subset S$ on the same reference
    probability measure, we extend $S$ by $K$ unboserved non-intrusion events
    $\overline S$ so that $S \circ \overline S$ contains exactly $N$
    non-intrusion events. Since intervals are mutually independent,
    the expected joint probability density of the unobserved events is the product of
    expected probability densities of each interval:
    \begin{equation}
        \E(\widetilde \Pr(\overline S)) = \E_{\tau \sim F}[f(\tau)]^{K}
       \label{eqn:e-prob-hat-s}
    \end{equation}
    $\widetilde \Pr(I|S)$ is computed as the product of
    probability densities of each transition in $S \setminus I$
    and of expected joint probability density of the unobserved events:
    \begin{align}
        \label{eqn:prob-i-tag}
         \widetilde \Pr&(I| S) = P'_{k_1} \!\!\!\!\! \prod_{j=1}^{N-K-1}\!\!\!\!\!Q'_{k_j, k_{j+1}} R'_{k_{N-K}} \E(\widetilde \Pr(\overline S)), \\ \nonumber
        \mbox  {where}& \\ \nonumber
         P_k &= \cdot
        \begin{cases}
            p_\epsilon^N  \frac {\int_{T}^\infty (\tau - T)f(\tau)^2 d\tau}{\int_{T}^\infty (\tau - T)f(\tau) d\tau} & \mathbf{if}\; k = N + 1, \\
            p_\epsilon^{k - 1} \frac {\int_{t_k - t_s}^\infty f(\tau)^2d\tau} {\int_{t_k - t_s}^\infty f(\tau)d\tau} & \mathbf{otherwise.}
        \end{cases} \\ \nonumber
         Q_{k_1,k_2}& = \cdot (1-p_\epsilon)p_\epsilon^{k_2 - k_1 - 1} {f(t_{k_2} - t_{k_1})}, \\ \nonumber
         R_k& = \cdot 
        \begin{cases}
            1 & \mathbf{if}\; k = 0, \\
            (1-p_\epsilon)p_\epsilon^{N - k} 
            \frac {\int_{t_e - t_k}^\infty f(\tau)^2d\tau}{\int_{t_e - t_k}^\infty f(\tau)d\tau} & \mathbf{otherwise.}
        \end{cases}
    \end{align}
    Here, $P_{N+1}$ accounts for the case when all events in $S$ are intrusion
    events. The probability density of a randomly chosen interarrival interval
    of duration $\Delta t$ is $\E_{\tau \sim F}[\tau]^{-1} \Delta t f(\Delta t)$ (known
    as `observation paradox'~\cite{C70}). The probability density of an
    interarrival interval of duration $\Delta t$ covering $[t_s, t_e]$ is
    $\propto (\Delta t - T)f(\Delta t)^2$ for $\Delta t >= T$, 0 otherwise.
    $P_k$ and $R_k$ for $k \in \{1, \dots, N\}$ account for intervals from
    first and last events in $S$ to the corresponding extra events. $Q_{k_1,
    k_2}$ account for intervals between events in $S$. 
    
    $\widetilde \Pr(I|S)$ is unnormalized, hence can be scaled by any
    factor that does not depend on $I$ or $S$.  Equation
    (\ref{eqn:prob-i}) is obtained as
    \begin{equation}
        \widetilde \Pr(I|S) \gets \frac {\widetilde \Pr(I|S)}
                                        {\E_{\tau \sim F}[f(\tau)]^{N + 1}}
        \label{eqn:prob-i-scaled}
     \end{equation}
\end{proof}

Theorem~\ref{thm:post-map}

\begin{proof} According to (\ref{eqn:prob-i}),
    \begin{equation}
        - \log \widetilde \Pr(I|S) = w_{0, k_1} + \sum_{j=1}^{N-K-1} w_{k_j, k_{j+1}} + w_{k_{N-k}, N+1}
        \label{eqn:post-map-proof-a}
    \end{equation}
    which is also the length of a path from $e_0$ to $e_{N+1}$.
    $\log(\cdot)$ is monotonically increasing,
    hence minimizing (\ref{eqn:post-map-proof-a}) computes
    $I_{MAP}$ through maximizing the posterior probability
    $\widetilde \Pr(I|S)$.
\end{proof}

\begin{crl}
    Provided that the probability density of $F(\theta_F)$ can
    be computed in fixed time, a MAP subsequence of intrusion
    events can be computed in time $\Theta(N^2)$.
    \label{crl:post-map-complexity}
\end{crl}

\begin{proof}
    Constructing $G_I$ requires $\Theta(N^2)$ for computing the
    edge weights according to (\ref{eqn:post-map-weights}).  The
    shortest path between two vertices in a directed acyclic
    graph $G=(V,E)$ can be computed in $\Theta(|V| + |E|)$ time
    ~\cite[Section 24.2]{CSR+01}. In $G_I$, $|V_I|=N+2$, $|E_I|=\frac
    {|V_I|(|V_I|-1)} 2$, hence a shortest path in $G_I$ can be
    found in $\Theta(N^2)$. Hence, the total computation time of
    $I_{MAP}$ is $\Theta(N^2)$.
\end{proof}

Lemma~\ref{lmm:post-ip-ml-alg}

\begin{proof} 
    The proof uses similar reasoning to the proof of
    Lemma~\ref{lmm:prob-i}. Any event $e_k$ belonging to the process
    can be reached from any event $e_j$ preceding it, $0 \le j <
    k$.  Line~\ref{alg:post-ip-ml-s} accounts for transitions from the extra
    event at the beginning to $e_k$.  Line~\ref{alg:post-ip-ml-j} --- for
    transitions from earlier events in $S \setminus I$ to $e_k$.  $a_k$,
    $\forall k \in \{1, \dots, N\}$, are the marginal likelihoods of
    subsequences $S_{1:k}$ over time intervals $[t_s, t_k]$.  Similarly,
    lines~\ref{alg:post-ip-ml-e-start}--\ref{alg:post-ip-ml-e-end} account for
    transitions from any event to the extra event at the end. $A$ is the
    marginal likelihood of $S$ over time interval $[t_s, t_e]$.

    The running time is dominated by the nested loop in
    lines~\ref{alg:post-ip-ml-loop-start}--\ref{alg:post-ip-ml-loop-end}.
    Line~\ref{alg:post-ip-ml-j} in the loop is executed
    $\frac {N(N-1)} 2$ times. Hence, the algorithm runs in time
    $\Theta(N^2)$.
\end{proof}

Theorem~\ref{thm:post-ip-all}

\begin{proof}
    A renewal process is a Markov process: the arrival time of
    an event is independent of earlier events given the last
    event. Consequently, the likelihood of a sequence stays the
    same if the times of events and the interval bounds
    are reversed. Therefore, $a_k^f$, $\forall k \in \{1, \dots,
    N\}$, are the marginal likelihoods of subsequences~$S_{1:k}$
    over time intervals~$[t_s, t_k]$, and $a_k^b$ are the
    marginal likelihoods of $S_{k:N+1}$ over~$[t_k, t_e]$.  The
    loop in
    lines~\ref{alg:post-ip-all-loop-start}--\ref{alg:post-ip-all-loop-end}
    computes the marginal probabilities of events in $S$ to
    belong to the intrusion. Line~\ref{alg:post-ip-all-once}
    computes the probability of $S$ with $e_k \notin I$ by
    multiplying $a_k^f$ and $a_k^b$ and dividing by the
    probability of $e_k \notin I$ independently of other events,
    because this probability appears twice, both in $a_k^f$ and
    in $a_k^b$.  The expressions for returned values in
    lines~\ref{alg:post-ip-all-ip} and~\ref{alg:post-ip-all-mp}
    are due to (\ref{eqn:post-ip-bayes-ip})
    and~(\ref{eqn:post-ip-bayes-mp}).

    Algorithm~\ref{alg:post-ip-ml} runs in time $\Theta(N^2)$
    and is called twice. The rest of the algorithm runs in time
    $\Theta(N)$. Hence, the algorithm runs in time
    $\Theta(N^2)$. 
\end{proof}

\end{document}